\documentclass[11pt]{article}

\usepackage{graphicx}
\usepackage{url}
\usepackage{arial}
\usepackage{setspace}
\usepackage{amsthm, amsmath, amssymb}
\usepackage{stmaryrd}
\usepackage{natbib}

\newtheorem{thm}{Theorem}

\newtheorem{cor}{Corollary}
\newtheorem{lemma}{Lemma}
\newtheorem{defn}{Definition}

\theoremstyle{plain}

\newcommand{\sX}{{\cal X}}

\newcommand{\sF}{{\cal F}}
\newcommand{\reals}{\mathbb{R}}
\newcommand{\ind}[1]{1_{\{#1\}}}
\newcommand{\Lp}{L_1}
\newcommand{\Lm}{L_{-1}}
\renewcommand{\a}{\alpha}
\newcommand{\HLa}{H_{L,\alpha}}
\newcommand{\CLam}{C_{L,\alpha}^-}
\newcommand{\topp}{t\in\reals:t(\eta-\a) \le 0}
\newcommand{\nuLa}{\nu_{L,\alpha}}
\newcommand{\muLa}{\mu_{L,\alpha}}
\newcommand{\psiLa}{\psi_{L,\alpha}}
\newcommand{\eps}{\epsilon}
\newcommand{\etaeq}{\eta \in[0,1]:|\eta-\a|=\eps}
\newcommand{\co}{\mathop{\mbox{\rm co}}}
\newcommand{\Epi}{\mathop{\mbox{\rm Epi}}}
\newcommand{\sign}{\mathop{\mathrm{sign}}}

\newcommand{\bd}{\begin{description}}
\newcommand{\ed}{\end{description}}
\newcommand{\beas}{\begin{eqnarray*}}
\newcommand{\eeas}{\end{eqnarray*}}
\newcommand{\vt}{\vartheta_\alpha(\eta)}

\newcommand{\sgnne}{\sign(t) \ne \sign(\eta-\a)}
\renewcommand{\d}{\delta}
\newcommand{\dt}{\tilde{\delta}}
\newcommand{\Cad}{C_\a(\eta,t)-C_\a^*(\eta)}
\newcommand{\CLd}{C_L(\eta,t)-C_L^*(\eta)}
\newcommand{\beq}{\begin{equation}}
\newcommand{\eeq}{\end{equation}}

\newcommand{\g}{\gamma}
\newcommand{\sfrac}[2]{\mbox{$\frac{#1}{#2}$}}
\newcommand{\HLaa}{H_{L_\a,\a}}
\newcommand{\nuLaa}{\nu_{L_\a,\a}}
\newcommand{\bcase}{\left\{ \begin{array}{ll} }
\newcommand{\ecase}{\end{array} \right. }
\newcommand{\third}{\mbox{$\frac13$}}
\newcommand{\fnhat}{\widehat{f}_n}

\newcommand{\half}{\mbox{$\frac12$}}

\begin{document}

\title{Calibrated Surrogate Losses for Classification with
Label-Dependent Costs}

\author{Clayton Scott \\ Department of Electrical Engineering and
Computer Science \\ Department of Statistics \\ University of Michigan,
Ann Arbor}


\maketitle

\begin{abstract}
We present surrogate regret bounds for arbitrary surrogate losses in the
context of binary classification with label-dependent costs. Such
bounds relate a classifier's risk, assessed with respect to a surrogate
loss, to its cost-sensitive classification risk. Two approaches to
surrogate regret bounds are developed.  The first is
a direct generalization of \citet{bartlett06}, who focus on margin-based
losses and cost-insensitive classification, while the second adopts the
framework of \citet{steinwart07} based on calibration functions.
Nontrivial surrogate regret bounds are shown to exist precisely when the
surrogate loss satisfies a ``calibration" condition that is easily
verified for many common losses. We apply this theory to the class of
uneven margin losses, and characterize when these losses are properly
calibrated.  The uneven hinge, squared error, exponential, and sigmoid
losses are then treated in detail.
\end{abstract}

\section{Introduction}

Binary classification is concerned with the prediction of a label $Y \in
\{-1,1\}$ from a feature vector $X$ by means of a classifier.  A
classifier can be represented as a mapping $x \mapsto \sign(f(x))$ where
$f$ is a real-valued decision function. The goal of classification is to
learn $f$ from a training sample $(X_1,Y_1), \ldots, (X_n,Y_n)$. When the
cost of misclassifying $X$ is not dependent on $Y$, the performance of $f$
is typically measured by the risk $R(f) = E_{X,Y}[\ind{Y \ne f(X)}]$.
Since minimization of the empirical risk is usually intractable, it is
common in practice to instead minimize the empirical version of the
$L$-risk $R_L(f) = E_{X,Y}[L(Y,f(X))]$, where $L(y,t)$ is a surrogate
loss, chosen for its computational qualities such as convexity.

\citet*{bartlett06} study conditions under which consistency with respect
to an $L$-risk implies consistency with respect to
the original risk $R(f)$.  To be more specific, let $R^*$ and $R_L^*$
denote the minimal risk and $L$-risk, respectively, over all possible
decision functions. \citeauthor{bartlett06} examine when there
exists an
invertible function $\theta$ with $\theta(0)=0$ such that
\begin{equation}
\label{eqn:theta}
R(f) - R^* \le \theta(R_L(f) - R_L^*)
\end{equation}
for all $f$ and all distributions on $(X,Y)$.  We refer to such a
relationship as a surrogate regret bound, since $R(f) - R^*$ and $R_L(f) -
R_L^*$ are known as the regret and surrogate regret, respectively.

\citeauthor{bartlett06} study margin losses, which have the form $L(y,t) =
\phi(yt)$ for some $\phi:\reals \to [0,\infty)$.  They show that
non-trivial surrogate regret bounds exist precisely when $L$ is
classification-calibrated, which is a technical condition they develop.

In this paper we extend the work of \citeauthor{bartlett06} in two ways.
First, we consider risks that account for label-dependent
misclassification costs.  Second, we study arbitrary surrogate losses, not
just margin losses.  We show that non-trivial surrogate regret bounds
exist when $L$ is $\a$-classification calibrated, where $\a \in (0,1)$
represents the misclassification cost asymmetry.  This condition is a
natural generalization of classification calibrated.  We also give results
that facilitate the calculation of these bounds, and characterization of
which losses are $\a$-classification calibrated.

\citet{steinwart07} extends the work of \citeauthor{bartlett06} in a
very general way that encompasses several supervised and unsupervised
learning problems.  He applies this framework to cost-sensitive
classification, but restricts his attention to margin-based losses. We
apply this framework to derive surrogate regret bounds for cost-sensitive
classification and arbitrary losses.  The results obtained in this manner
are shown to be equivalent to the bounds obtained by generalizing the
approach of \citeauthor{bartlett06}

\sloppy \citet{reid09icml, reid09tr} also study $\a$-classification
calibrated losses and derive surrogate regret bounds for cost-sensitive
classification. Their focus is on class probability estimation, and unlike
the present work, they impose certain conditions on the surrogate loss,
such as differentiability everywhere. Therefore they do not address
important losses such as the hinge loss. In addition, their bounds are not
in the form of (\ref{eqn:theta}), but rather are stated implicitly. We
also note that their examples of surrogate regret bounds
\citep{reid09icml} consider only margin losses.

Additional comparisons to the above cited and other works are given later.
Because we allow for asymmetry in both the misclassification costs and
surrogate loss, unlike the original analysis of \citet{bartlett06},
certain aspects of our analysis are necessarily different.

A motivation for this work is to understand uneven margin losses, which
have the form
$$
L(y,t) = \ind{y=1}\phi(t) + \ind{y=-1} \beta \phi(-\g t)
$$
for some $\phi:\reals \to [0,\infty)$ and $\beta, \g > 0$.  Various
instances of such losses have appeared in the literature (see Sec.
\ref{sec:uml} for specific references), primarily as a heuristic
modification of margin losses to account for cost asymmetry or unbalanced
datasets.  They are computationally attractive because they can typically
be optimized by modifications of margin-based algorithms.  However,
statistical aspects of these losses have not been studied.  We
characterize when they are $\a$-classification calibrated and compute
explicit surrogate regret bounds for four specific examples of $\phi$.

When applied to uneven margin losses, our work has practical implications 
for adapting well-known algorithms, such as Adaboost and support vector 
machines, to settings with unbalanced data or label-dependent costs. These 
are discussed in the concluding section.

The rest of the paper is organized as follows.  Section \ref{sec:bnd}
develops a general framework for surrogate regret bounds that handles
label-dependent costs and arbitrary surrogate losses.  The special case of
cost-insensitive classification with general losses is considered, and a
refined treatment is also given for the case of convex losses.  Section
\ref{sec:cal} relates our problem to the general framework of
\citet{steinwart07}, and provides an alternate, yet ultimately equivalent
approach to surrogate regret bounds using so-called calibration functions.
Section \ref{sec:uml} examines uneven margin losses in detail, including
four specific instances of $\phi$ corresponding to the hinge, squared
error, exponential, and sigmoid functions. A concluding discussion is
offered in Section \ref{sec:disc}. Supporting lemmas and additional
details may be found in two appendices.

\section{Surrogate Losses and Regret Bounds}
\label{sec:bnd}

Let $(X,Y)$ have distribution $P$ on $\sX\times\{-1,1\}$.  Let $\sF$
denote the set of all measurable functions $f: \sX\to\reals$.  Every
$f\in\sF$ defines a classifier by the rule $x\mapsto \sign(f(x))$.  We
adopt the convention $\sign(0) = -1$

A loss is a measurable function $L:\{-1,1\}\times\reals \to[0,\infty)$.
Any loss can be written
$$
L(y,t) = \ind{y=1}\Lp(t)+\ind{y=-1}\Lm(t).
$$
We refer to $\Lp$ and $\Lm$ as the partial losses of $L$.  The $L$-risk of
$f$
is $R_L(f): = E_{X,Y}[L(Y,f(X))]$.  The optimal $L$-risk is
$R_L^*:=\inf_{f\in\sF} R_L(f)$.
The cost-sensitive classification loss with cost parameter $\a\in(0,1)$ is
$$
U_\a(y,t):=(1-\a)\ind{y=1}\ind{t\le0} +\a\ind{y=-1}\ind{t>0}.
$$

When $L=U_\a$, we write $R_\a(f)$ and $R_\a^*$ instead of
$R_{U_\a}(f)$ and $R_{U_\a}^*$. Although other parametrizations of
cost-sensitive classification losses are possible, this one is
convenient because an optimal classifier is $\sign(\eta(x)-\a)$
where $\eta(x):=P(Y=1 |X=x)$. See Lemma \ref{lemmaX}, part 1. We are
motivated by applications where it is desirable to minimize the
$U_\a$-risk, but the empirical $U_\a$-risk cannot be optimized
efficiently.  In such situations it is common to minimize the
(empirical) $L$-risk for some surrogate loss $L$ that has a
computationally desirable property such as differentiability or
convexity.

Define the conditional $L$-risk
$$
C_L(\eta,t):=\eta\Lp(t) + (1-\eta)\Lm(t)
$$
for $\eta\in[0,1],t\in\reals$, and the optimal conditional $L$-risk
$C_L^*(\eta)=\inf_{t\in\reals}C_L(\eta,t)$ for $\eta\in[0,1]$. These are
so-named because
$R_L(f) = E_X[C_L(\eta(X),f(X))]$ and
$R_L^*(\eta)= E_X[C_L^*(\eta(X))]$. Note that we use $\eta$ to denote both
the function $\eta(x) = P(Y=1|X=x)$ and a scalar $\eta \in [0,1]$. The
meaning should be clear from context. When $L=U_\a$, we write
$C_\a(\eta,t)$ and $C_\a^*(\eta)$ for $C_{U_\a}(\eta,t)$
and $C_{U_\a}^*(\eta)$. Measurability issues with these and other
quantities are addressed in \cite{steinwart07}.

This section has three parts.  In \ref{sec:nu} we extend the work of
\citet{bartlett06}, on surrogate regret bounds for margin losses and
cost-insensitive classification, to general losses and cost-sensitive
classification.  In \ref{sec:costins} we specialize our results to the
important special case of cost-insensitive classification with general
losses, and in \ref{sec:conv} we present some results for the case of
convex partial losses.

\subsection{$\a$-classification calibration and surrogate regret bounds}
\label{sec:nu}

For $\a\in(0,1)$ and any loss $L$, define
$$
\HLa(\eta):=\CLam(\eta)-C_L^*(\eta)
$$
for $\eta\in[0,1]$, where
$$
\CLam(\eta):=\inf_{\topp} C_L(\eta,t).
$$
Note that $\HLa(\eta)\ge0$ for all $\eta\in [0,1]$.
\begin{defn}
We say $L$ is {\em $\a$-classification calibrated}, and write $L$ is
$\a$-CC, if
$\HLa(\eta)>0$ for all $\eta\in [0,1], \eta \ne\a$.
\end{defn}
Intuitively, $L$ is $\a$-CC if, for all $x$ such that $\eta(x) \ne\a$, the
value of $t=f(x)$ minimizing the conditional $L$-risk has the same sign as the optimal predictor $\eta(x) - \a$.
Denote $B_\a:=\max(\a,1-\a)$.  Note that the  regret, $R_\a(f)-R_\a^*$,
and the conditional regret, $C_\a(\eta,t)-C_\a^*(\eta)$, both take value
in
$[0,B_\a]$.  This can be seen from Lemma \ref{lemmaX}, part 1.
Next, define
$$
\nuLa(\eps) = \min_{\etaeq}\HLa(\eta)
$$
for $\eps\in[0,B_\a]$.  Notice that for $\a\le\half$,
\beq
\label{eqn:nu1}
\nuLa(\eps)=\left\{
\begin{array}{ll}
\min(\HLa(\a+\eps),\HLa(\a-\eps)), & 0\le \eps\le\a\\
\HLa(\a+\eps),&\a<\eps\le 1-\a
\end{array}
\right.
\eeq
and for $\a\ge\half$,
\beq
\label{eqn:nu2}
\nuLa(\eps)=\left\{
\begin{array}{ll}
\min(\HLa(\a+\eps),\HLa(\a-\eps)), & 0\le\eps\le 1-\a \\
\HLa(\a-\eps), & 1-\a<\eps\le\a.
\end{array}
\right.
\eeq
Finally, define $\psiLa(\eps)=\nuLa^{**}(\eps)$ for $\eps\in[0,B_\a]$,
where
$g^{**}$ denotes the Fenchel-Legendre biconjugate of $g$.  The biconjugate
of $g$ is the largest lower semi-continuous function that is $\le g$, and
is defined by
$$
\Epi g^{**}=\overline{\co\Epi g},
$$
where $\Epi g=\{(r,s):g(r)\le s\}$ is the epigraph of $g$, $\co$ denotes
the
convex hull, and the bar indicates set closure.
Since $\nuLa(0)=0$ (Lemma \ref{lemmaX}, part 4), $\nuLa$ is nonnegative,
and
$\psiLa$ is convex, we know $\psiLa(0) = 0$ and $\psiLa$ is nondecreasing.

\begin{thm}
\label{thm:nu}
Let $L$ be a loss and $\a \in (0,1)$.
\bd
\item[1.] For all $f \in \sF$ and all distributions $P$,
$$
\psiLa (R_\a(f)-R_\a^*) \le R_L(f)-R_L^*.
$$
\item [2.] $\psiLa$ is invertible if and only if $L$ is $\a$-CC.
\ed
\end{thm}

\begin{proof}
For the first part, by Lemma \ref{lemmaX} part 1 we
know
\beas
R_\a(f)-R_\a^* &=& E_X[\ind{\sign f(X)
\ne \sign (\eta(X)-\a)}|\eta(X)-\a|] \\
&\le& E_X[\ind{f(X) (\eta(X)
-\a) \le 0}|\eta(X)-\a|].
\eeas
Then
\beas
\lefteqn{\nuLa^{**}(R_\a(f)-R_\a^*)\le
E_X[\nuLa^{**}(\ind{f(X)(\eta(X)-\a)\le0}
|\eta(X)-\a|)]} \\
&\le& E_X[\nuLa(\ind{f(X)(\eta(X)-\a)\le0}
|\eta(X)-\a|)] \\
&=& E_X[\ind{f(X)(\eta(X)-\a)\le0}\nuLa(
|\eta(X)-\a|)] \\
&=& E_X \left[\ind{f(X)(\eta(X)-\a)\le0}
\min_{\eta'\in[0,1]:|\eta'-\a|=|\eta(X)-\a|}
\HLa(\eta') \right] \\
&\le& E_X[\ind{f(X)(\eta(X)-\a)\le0}
\HLa (\eta(X))] \\
&=& E_X \left[ \ind{f(X)(\eta(X)-\a)\le0}
\left(\inf_{t:t(\eta(X)-\a)\le0} C_L(\eta(X),
t) - C_L^*(\eta(X)) \right) \right] \\
&\le& E_X[C_L(\eta(X),f(X))-C_L^*(\eta(X))]
\\
&=& R_L(f)-R_L^*.
\eeas
The first inequality is Jensen's, and the first equality follows from
$\nuLa(0) = 0$.

Now consider the second part.  If $\psiLa$ is invertible,
then $\psiLa(\eps)>0$ for all $\eps \in [0,B_\a]$,
because $\psiLa(0)=0$ and $\psiLa$ is nonnegative.
Since $\psiLa \le \nuLa$, we know $\nuLa(\eps)>0$ for all $\eps \in
(0,B_\a]$, which by definition of $\nuLa$ implies $\HLa (\eta) > 0 $ for
all $\eta\ne\a$. Thus $L$ is $\a$-CC.

Conversely, now suppose $L$ is $\a$-CC. We claim that $\psiLa(\eps) > 0 $
for all $\eps \in (0,B_\a]$. To see this, suppose $\psiLa(\eps) = 0$.
Since $\nuLa$ is lower semi-continuous, $\Epi \nuLa$ and $\co \Epi \nuLa$
are closed sets.  Therefore, $(\eps,0)$ is a convex combination of points
in $\Epi \nuLa$. Since $L$ is $\a$-CC, we know $\nuLa (\eps) > 0 $ for all
$\eps \in (0,B_\a]$. Therefore $\eps = 0 $. This proves the claim.

Since $\psiLa(0) = 0$ and $\psiLa$ is convex and nondecreasing, it follows
that $\psiLa$ is strictly increasing. Since $\psiLa$ is continuous (Lemma
\ref{lemmaX}, part 5), we conclude that $\psiLa$ is invertible.
\end{proof}

If $L$ is $\a$-CC, then $R_\a(f)-R_\a^* \le \psiLa^{-1} (R_L(f)-R_L^*)$.
Since $\psiLa(0)=0$ and $\psiLa$ is nondecreasing, the same is true of
$\psiLa^{-1}$. As a result, we can show that an algorithm that is
consistent for the $L$-risk is also consistent for the $\a$ cost-sensitive
classification risk. Such an approach was employed by \citet{zhang04} and
\citet{steinwart05} to prove consistency, for the cost-insensitive risk,
of different algorithms based on surrogate losses.

\begin{cor}
Suppose $L$ is $\a$-CC.
\bd
\item[1.] If $R_L(f_i)-R_L^* \to 0$ for some sequence of decision
functions $f_i$, then $R_\a(f_i)-R_\a^* \to 0$.
\item[2.] Let $\fnhat$ be a classifier based on the random sample
$(X_1,Y_1), \ldots, (X_n,Y_n)$. If $R_L(\fnhat)-R_L^* \to 0$ in
probability,
then $R_\a(\fnhat)-R_\a^* \to 0$ in probability. If $R_L(\fnhat)-R_L^* \to
0$
with probability one, then $R_\a(\fnhat)-R_\a^* \to 0$ with probability
one.
\ed
\end{cor}

\begin{proof}
Since $L$ is $\a$-CC, $\psiLa$ is invertible. For any $\eps \in (0,B_\a]$,
if
$R_L(f)-R_L^* < \psiLa(\eps)$, then $R_\a(f)-R_\a^* \le
\psiLa^{-1}(R_L(f)-R_L^*) < \eps$. Now 1 follows.

Assume $R_L(\fnhat)-R_L^* \to 0$ in probability. By the above reasoning,
if $R_\a(f)-R_\a^* \ge \eps$, then $R_L(f)-R_L^* \ge \psiLa(\eps)$.
Therefore, for any $\eps \in (0,B_\a]$,
$$
P(R_\a(\fnhat)-R_\a^* \ge \eps) \le P(R_L(\fnhat)-R_L^* \ge \psiLa(\eps))
\to 0
$$
as $n \to \infty$ by assumption.

Assume $R_L(\fnhat)-R_L^* \to 0$ with probability one. By part 1,
$$
P \left( \lim_{n \to \infty} R_\a(\fnhat)-R_\a^* = 0 \right) \ge
P \left( \lim_{n \to \infty} R_L(\fnhat)-R_L^* = 0 \right) = 1.
$$
Hence $R_\a(\fnhat)-R_\a^* \to 0$ with probability one.
\end{proof}

Below in Section \ref{sec:uml}, the above results are made more concrete
when we examine some specific losses (namely, uneven margin losses).

\subsection{Cost-insensitive classification}
\label{sec:costins}

We turn our attention to the cost-insensitive or 0/1 loss,
$$
U(y,t):=\ind{y=1}\ind{t\le0}+\ind{y=-1}\ind{t>0}
=2U_{1/2}(y,t).
$$
This loss is not only important in its own right,
but the associated quantity $H_L$, defined below, is
useful for calculating $\HLa$ when $\a \ne \frac12$,
as explained below.
The results in this section generalize those of \cite{bartlett06},
who focus on margin losses.  We place no restrictions on
the partial losses $\Lp$ and $\Lm$.

For an arbitrary loss L, define
$$
H_L(\eta):=C_L^-(\eta) - C_L^*(\eta)
$$
for $\eta \in[0,1]$, where
$$
C_L^-(\eta):=\inf_{t:t(2\eta-1)\le0} C_L(\eta,t).
$$
Also define for $\eps\in[0,1]$
\beas
\nu_L(\eps)&:=&\min_{\eta\in[0,1]:|2\eta-1|=
\eps}H_L(\eta) \\
&=&\min\{H_L(\mbox{$\frac{1+\eps}2$}),
H_L(\mbox{$\frac{1-\eps}2$}) \}.
\eeas
Finally, define $\psi_L(\eps)=\nu_L^{**}(\eps)$
for $\eps \in [0,1]$.

The following definition was introduced by \cite{bartlett06}
in the context of margin losses.

\begin{defn}
If $H_L(\eta)>0$ for all $\eta \in [0,1],
\eta \ne \frac12$, $L$ is said to be {\em classification
calibrated}, and we write $L$ is CC.
\end{defn}

For margin losses, this coincides with the definition of
\citeauthor{bartlett06}, and our $H_L$ equals their $\tilde{\psi}$. Also
note that $H_L(\eta)=H_{L,1/2}(\eta)$, and therefore $L$ is CC iff $L$ is
$\frac12$-CC. When $L=U$, we write $R(f), R^*, C(\eta,t)$, and $C^*(\eta)$
instead of
$R_U(f), R_U^*, C_U(\eta,t)$, and $C_U^*(\eta)$, respectively.

\begin{thm}
Let $L$ be a loss.
\bd
\item[1.] For any $f \in \sF$ and any
distribution $P$,
$$
\psi_L(R(f)-R^*) \le R_L(f)-R_L^*.
$$
\item[2.] $\psi_L$ is invertible if and only if $L$ is CC.
\ed
\end{thm}

\begin{proof}
The proof follows from Theorem \ref{thm:nu} and the relationships $C(\eta,t)=
2C_{1/2}(\eta,t)$, $C^*(\eta)=2C_{1/2}^*(\eta)$,
$H_L(\eta)=H_{L,1/2}(\eta)$, $\nu_L(\eps)=
\nu_{L,1/2}(\mbox{$\frac{\eps}2$})$, and $\psi_L(\eps)=
\psi_{L,1/2}(\mbox{$\frac{\eps}2$})$. Thus,
to prove 1, note
\beas
\psi_L(R(f)-R^*) &=& \psi_{L,1/2}
(\half E_X[C(\eta(X),f(X))-C^*(\eta(X))]) \\
&=& \psi_{L,1/2}(E_X[C_{1/2}(\eta(X),f(X))
- C_{1/2}^*(\eta(X))] \\
&=& \psi_{L,1/2}(R_{1/2}(f)-R_{1/2}^*) \\
&\le& R_L(f)-R_L^*.
\eeas
To prove 2, note $\psi_L$ is invertible $\iff$
$\psi_{L,1/2}$ is invertible $\iff$ $L$ is
$\half$-CC $\iff$ $L$ is CC.
\end{proof}

When $L$ is a margin loss, $H_L$ is symmetric with respect to $\eta =
\frac12$, and the above result reduces to the surrogate regret bound
established by \citet{bartlett06}.

The following extends a result for margin losses noted by
\citet{steinwart07}.
For any loss $L$, we can express
$\HLa$ in terms of $H_L$.
This simplifies the determination of $\HLa, \nuLa$, and $\psiLa$.

Given the loss $L(y,t)=\ind{y=1}\Lp(t)+\ind{y=-1}
\Lm(t)$ and  $\a \in (0,1)$ define
$$
L_\a(y,t):=(1-\a)\ind{y=1}\Lp(t)+\a\ind{y=-1}
\Lm(t).
$$
Also introduce $w_\a(\eta)=(1-\a)\eta+\a(1-\eta)$
and
$$
\vartheta_\a(\eta)=\frac{(1-\a)\eta}{(1-\a)\eta+
\a(1-\eta)}.
$$

\begin{thm}
\label{thm:HLaa}
For any loss $L$ and any $\a \in (0,1)$,
\bd
\item[1.] For all $\eta \in [0,1]$,
\begin{equation}
\label{eqn:HLaa}
H_{L_\a,\a}(\eta)=w_\a(\eta)H_L(\vartheta_\a
(\eta)).
\end{equation}
\item[2.] $L$ is CC $\iff$ $L_\a$ is $\a$-CC.
\item[3.] $L$ is $\a$-CC $\iff$ $L_{1-\a}$ is CC.
\ed
\end{thm}

\begin{proof}
Notice that $w_\a(\eta)>0$ for all $\eta \in [0,1]$,
and $2\vartheta_\a(\eta)-1=(\eta-\a)/w_\a(\eta)$.
Thus $\sign(2\vartheta_\a(\eta)-1)=\sign(\eta-\a)$.
In addition, $\vartheta_\a:[0,1]\to [0,1]$ is a
bijection. To prove 1, observe
\beas
C_{L_\a}(\eta,t)&=&(1-\a)\eta\Lp(t)+
\a(1-\eta)\Lm(t) \\
&=& w_\a(\eta)[\vartheta_\a(\eta)\Lp(t)+
(1-\vartheta_\a(\eta))\Lm(t)] \\
&=& w_\a(\eta)C_L(\vartheta_\a(\eta),t).
\eeas
Therefore $C_{L_\a}^*=w_\a(\eta) C_L^*
(\vartheta_\a(\eta))$ and
\beas
\CLam(\eta)&=& \inf_{\topp} C_{L_\a}(\eta,t) \\
&=& w_\a(\eta) \inf_{t:t(2\vartheta_\a(\eta)-1)\le0}
C_L(\vartheta_\a(\eta),t) \\
&=& w_\a(\eta) C_L^-(\vartheta_\a(\eta),t).
\eeas
Therefore
\beas
H_{L_\a,\a}(\eta) &=& \CLam(\eta)-C_{L_\a}^*(\eta) \\
&=& w_\a(\eta) [C_L^-(\vt)-
C_L^*(\vt)] \\
&=& w_\a(\eta) H_L(\vt).
\eeas
The second statement follows from 1, the positivity of
$w_\a$, and the fact that $\vartheta_\a$ is a
bijection with $\vartheta_\a(\a)=\half$.

To prove the third statement, notice $(L_{1-\a})_\a
=\a(1-\a)L$. Therefore, $L$ is $\a$-CC
$\iff$ $\a(1-\a)L$ is $\a$-CC $\iff$
$(L_{1-\a})_\a$ is $\a$-CC $\iff$ $L_{1-\a}$
is CC, where the last equivalence follows from 2.
\end{proof}

\subsection{Convex partial losses}
\label{sec:conv}

When the partial losses $\Lp$ and $\Lm$ are convex,
we can deduce some convenient characterizations of
$\a$-CC losses.

\begin{thm}
\label{thm:diff}
Let $L$ be a loss and $\a \in (0,1)$. Assume
$\Lp$ and $\Lm$ are convex
and differentiable at 0. Then $L$ is $\a$-CC
if and only if
\begin{equation}
\label{eqn:diff}
\Lp'(0)<0, \Lm'(0)>0, \ \mbox{and} \
\a \Lp'(0) + (1-\a)\Lm'(0)=0
\end{equation}
\end{thm}

A similar result appears in \citet{reid09tr}, and when the loss is a
composite proper loss the results are equivalent. Their result is
expressed in the context of class probability estimation, while our result
is tailored directly to classification. Although the proofs are
essentially the same, our setting allows us to state the result without
assuming the loss is differentiable everywhere. Thus, it encompasses
losses that are not suitable for class probability estimation, such as the
uneven hinge loss described below. We also make an observation in the
special case where $\a = \frac12$ and $L$ is a margin loss, also noted by
\citet{reid09tr}. Then $\Lp'(0)=\phi'(0)$ and $\Lm'(0)=-\phi'(0)$, and
(\ref{eqn:diff}) is equivalent to $\phi'(0)<0$, the condition identified
by \cite{bartlett06}.

\begin{proof}
Note that $\frac{\partial}{\partial t} C_L(\eta,0)
=\eta\Lp'(0) + (1-\eta)\Lm'(0)$. Now
$L$ is $\a$-CC if and only if $\CLam(\eta)>
C_L^*(\eta)$ for all $\eta \in [0,1], \eta \ne \a$,
and by convexity of $\Lp$ and $\Lm$, the latter condition
holds if and only if
\begin{equation}
\label{eqn:line}
\eta \Lp'(0) + (1-\eta)\Lm'(0) \left\{
\begin{array}{ll}
<0& \mbox{ if } \eta > \a \\
>0& \mbox{ if } \eta < \a
\end{array}
\right..
\end{equation}

Thus, we must show (\ref{eqn:diff}) $\iff$
(\ref{eqn:line}). Assume (\ref{eqn:line}) holds.
Since $\eta \mapsto \eta \Lp'(0) + (1-\eta)
\Lm'(0)$ is continuous, we must have $\a\Lp'(0)
+ (1-\a)\Lm'(0)=0$. $\Lp'(0)<0$ follows from
(\ref{eqn:line}) with $\eta=1$, and $\Lm'(0)>0$
follows from (\ref{eqn:line}) with $\eta=0$.

Now suppose (\ref{eqn:diff}) holds. Then $\eta
\mapsto \eta \Lp'(0) + (1-\eta)\Lm'(0)$ is
an affine function with negative slope that outputs 0
when $\eta=\a$. Thus (\ref{eqn:line}) holds.
\end{proof}
The following result facilitates calculation of regret bounds.

\begin{thm}
\label{thm:conv}
Assume $\Lp$ and $\Lm$ are convex.
\bd
\item[1.] If $L$ is $\a$-CC, then $\CLam(\eta)
=\eta \Lp(0) + (1-\eta)\Lm(0)$ and $\HLa$ is convex.
\item[2.] If $L$ is CC, then $C_L^-(\eta)=\eta
\Lp(0) + (1-\eta)\Lm(0)$, and $H_L$ is convex.
\ed
\end{thm}

\begin{proof}
The formulas for $\CLam$
and $C_L^-$ follow from definitions and
convexity of $\Lp$ and $\Lm$. $\HLa(\eta)
= \CLam(\eta) - C_L^*(\eta)$ is convex
because $\CLam$ is affine and $C_L^*$ is
concave (Lemma \ref{lemmaX}, part 2).
Therefore $H_L = H_{L,1/2}$ is also
convex.
\end{proof}

\section{Calibration Functions}
\label{sec:cal}

In this section we present an alternative, though ultimately equivalent,
approach to surrogate regret bounds. Additional properties of $\a$-CC
losses are derived, and connections to \citep{steinwart07} are
established. We begin with an alternate definition of $\a$-classification
calibrated.

\begin{defn} We say {\em $L$ is $\a$-CC'} if, for all
$\eps > 0, \eta \in[0,1]$, there exists $\delta > 0$ such that
\beq \label{eqn:cal}
\CLd < \d \implies \Cad < \eps.
\eeq
We say {\em $L$ is uniformly $\a$-CC'} if, for all
$\eps > 0$, there exists $\delta > 0$ such that
\beq \label{eqn:unifcal}
\forall \eta \in [0,1], \CLd < \d \implies \Cad < \eps.
\eeq
\end{defn}

Recall $B_\a = \max(\a,1-\a)$. For $\eps \in [0,B_\a]$ also define
$$
\muLa(\eps): = \inf_{\eta \in [0,1]: |\eta - \a| \ge \eps}
\HLa(\eps) = \inf_{\eps \le \eps' \le B_\a} \nuLa(\eps').
$$

\begin{thm}
\label{thm:cal}
Let $\a \in (0,1)$. For any loss $L$,
\bd
\item[1.] For all $\eps > 0, \eta \in [0,1]$
$$
\CLd < \HLa(\eta) \implies \Cad < \eps.
$$
\item[2.] For all $\eps > 0, \eta \in [0,1]$,
$$
\CLd < \muLa(\eps) \implies \Cad < \eps.
$$
\ed
If $L$ is $\a$-CC, then
\bd
\item[3.] $L$ is $\a$-CC'
\item[4.] $L$ is uniformly $\a$-CC'.
\ed
\end{thm}

\begin{proof}
To prove 1, let $\eps > 0, \eta \in [0,1]$. In
Lemma \ref{lemmaX}, part 1 it is shown that
$\Cad = \ind{\sgnne}|\eta-\a|$. Thus, if $\eps
>|\eta-\a|$, the result follows. Suppose $\eps \le
|\eta-\a|$. Then $\Cad \ge \eps \iff \sgnne$, and
\beas
\HLa(\eta) &=& \inf_{\topp} \CLd \\
&\le& \inf_{t:\sgnne} \CLd \\
&=& \inf_{t: \Cad \ge \eps} \CLd.
\eeas
Therefore, if $\CLd < \HLa(\eta)$, then $\Cad < \eps$.

To prove 2, let $\eps > 0, \eta \in [0,1]$. If
$\eps > |\eta-\a|$, then as in part 1 the
result follows immediately. If $\eps \le |\eta-\a|$,
then $\muLa(\eps) \le \HLa(\eta)$ and the result
follows from part 1.

Since uniformly $\a$-CC' implies
$\a$-CC', 3 follows from 4. To show 4,
let $\eps > 0$. By Lemma \ref{lemmaX}, part 3,
$\HLa$ is continuous on $\{\eta \in [0,1]:|\eta-\a|
\ge \eps \}$. Thus for $\eps \le B_\a$, $\muLa(\eps)$
is the infimum of a continuous, positive function on a compact
set and therefore positive. Taking $\d = \muLa(\eps)$,
the result follows by part 2. If $\eps > B_\a$,
the result holds because $\Cad = \ind{\sgnne}
|\eta-\a| \in [0,B_\a]$.
\end{proof}
\citet{steinwart07} employs $\a$-CC' as the definition of classification
calibrated in the case of cost-sensitive classification.  Although $\a$-CC
implies $\a$-CC', the reverse implication is not true as
the counterexample $L = U_\a$ demonstrates (perhaps
ironically). Under a mild
assumption on the partial losses, Steinwart's definitions and ours
agree. This is part 1 of the following result.  Under this same mild
assumption, we can also express what \citeauthor{steinwart07} calls the
calibration function and uniform calibration function of $L$. These are
the quantities $\d(\eps,\eta)$ and $\d(\eps)$ in parts 2 and 3,
respectively.
\begin{thm} Assume $\Lp$ and $\Lm$ are continuous at $0$. \bd
\item[1.] The following are equivalent:
\bd
\item[(a)] $L$ is $\a$-CC
\item[(b)] $L$ is $\a$-CC'
\item[(c)] $L$ is uniformly $\a$-CC'
\ed
\item[2.] For any $\eps > 0$ and
$\eta \in [0,1]$, the largest $\d$ such that
(\ref{eqn:cal}) holds is
\begin{equation}
\label{eqn:calfun}
\d(\eps,\eta): = \left\{
\begin{array}{ll}
\infty, & \eps > |\eta-\a|, \\
\HLa(\eta), & \eps\le|\eta-\a|.
\end{array} \right.
\eeq
\item[3.] For any $\eps > 0$, the largest $\d$
such that (\ref{eqn:unifcal}) holds is
\beq
\label{eqn:unifcalfun}
\d(\eps): = \left\{
\begin{array}{ll}
\infty, & \eps > B_\a, \\
\muLa(\eps), & \eps \le B_\a.
\end{array}
\right.
\eeq
\ed
\end{thm}

\begin{proof}
We have already shown (a) implies (b) and (c),
and (c) implies (b) is obvious, so let us show
(b) implies (a).

If $\eps > 0$ and $\eta \in [0,1]$ are such
that $\eps \le |\eta-\a|$, then $\eta \ne \a$,
and under the continuity assumption we have
$$
\inf_{\topp} C_L(\eta,t) = \inf_{t:\sgnne}
C_L(\eta,t).
$$
Therefore, from the proof of Theorem \ref{thm:cal}, part 1,
\beq
\label{eqn:calfun2}
\HLa(\eta) = \inf_{t:\Cad \ge \eps} \CLd.
\eeq
Now assume (b) holds, and let $\eta \in [0,1]$,
$\eta \ne \a$. Set $\eps = |\eta-\a|$.
Since $L$ is $\a$-CC', the right hand side of
(\ref{eqn:calfun2}) is positive. Therefore $\HLa(\eta)
> 0$ which establishes (a).

Now consider part 2. If $\eps > |\eta-\a|$,
then $\Cad = \ind{\sgnne}|\eta-\a| < \eps$
regardless of $\d$. If $\eps \le |\eta-\a|$,
then (\ref{eqn:calfun2}) holds which establishes
the result in this case.

To prove 3, first consider
$\eps > B_\a$. Then $\Cad \le B_\a < \eps$
regardless of $\d$. Now suppose $\eps \le B_\a$.
Then $\{\eta \in [0,1]:|\eta-\a| \ge \eps \}$
is nonempty, and this case now follows from
part 2 and the definition of $\muLa$.
\end{proof}

An emphasis of \citet{steinwart07} is the relationship between surrogate
regret bounds and uniform calibration functions. In our setting, Theorem
\ref{thm:cal} part 2 directly implies a surrogate regret bound in terms of
$\muLa$.

\begin{thm}
\label{thm:mubnd}
Let $L$ be a loss, $\a \in (0,1)$. Then
$$
\muLa^{**}(R_\a(f) - R_\a(f)) \le R_L(f) - R_L^*.
$$
\end{thm}
This result is similar to Theorem 2.13 of
\citet{steinwart07} and surrounding discussion.
While that
result holds in a very general setting that spans
many learning problems,
Theorem \ref{thm:mubnd} specializes the underlying
principle to cost-sensitive classification.
\begin{proof}
By Theorem \ref{thm:cal}, part 2, we know that
$\CLd < \muLa(\eps) \implies \Cad < \eps$.
Given $f \in \sF$ and $x \in \sX$, let
$\eps = C_\a(\eta(x),f(x)) - C_\a^*(\eta(x))$.
Then $C_L(\eta(x),f(x)) - C_L^*(\eta(x)) \ge
\muLa(\eps)$, or in other words
$$
\muLa(C_\a(\eta(x),f(x)) - C_\a^*(\eta(x))) \le
C_L(\eta(x),f(x)) - C_L^*(\eta(x)).
$$
By Jensen's inequality,
\beas
\muLa^{**}(R_\a(f) - R_\a^*) &\le&
E_X[\muLa^{**}(C_\a(\eta(X),f(X) - C_\a^*(\eta(X)))]
\\
&\le& E_X[\muLa(C_\a(\eta(X),f(X))
- C_\a^*(\eta(X)))] \\
&\le& E_X[C_L(\eta(X),f(X)) - C_L^*(\eta(X))]
\\
&=& R_L(f) - R_L^*.
\eeas
\end{proof}
Thus, for any loss we have two surrogate regret bounds.
In fact, the two bounds are the same.
\begin{thm}
Let $\a \in (0,1)$.
\bd
\item[1.] For any loss $L$, $\muLa^{**} =
\nuLa^{**}$.
\item[2.] If $\Lp$ and $\Lm$ are convex, then
$\muLa = \nuLa$.
\ed
\end{thm}

\begin{proof}
Part 1 follows from Lemma \ref{lemmaZ}.
To see the second statement, recall that
$\HLa$ is nonnegative, $\HLa(\a) = 0$
(Lemma \ref{lemmaX},
part 4), and $\HLa$ is convex (Theorem
\ref{thm:conv}). Thus $\HLa(\eta)$ is
nondecreasing as $|\eta-\a|$ grows, and the
result follows.
\end{proof}

Thus $\nuLa$ and $\muLa$ give two approaches
to the same bound. $\nuLa$ is perhaps simpler
to conceptualize, and $\muLa$ is
connected to the notion of uniform calibration.

\section{Uneven Margin Losses}
\label{sec:uml}

We now apply the preceding theory to a special class of
asymmetric losses.
\begin{defn}
\label{def:uml}
Let $\phi: \reals \to [0, \infty)$ and $\beta,
\g > 0$. We refer to the losses
$$
L(y,t) = \ind{y=1}\phi(t) +  \ind{y=-1}
\beta\phi(-\g t)
$$
and
$$
L_\a(y,t) = (1-\a)\ind{y=1}\phi(t) + \a
\ind{y=-1}\beta\phi(-\g t)
$$
as {\em uneven margin losses}.
\end{defn}
When $\beta = \g = 1$, $L$ in Definition \ref{def:uml}
is a conventional margin loss, and $L_\a$ can be called an
$\a$-weighted margin loss.
Since they differ from margin losses
by a couple of scalar parameters, empirical risks based on uneven
margin losses can typically be
optimized by slightly modified versions of margin-based
algorithms.

Before proceeding, we offer a couple of comments on
Definition \ref{def:uml}. First, although $\beta$ may
appear redundant in $L_\a$, it is not.
$\a$ is fixed at a desired cost parameter, and
thus is not tunable. Second, there would be no
added benefit from a loss of the form
$\ind{y=1}\phi(\g' t) + \ind{y=-1}\beta \phi(-\g t)$. We may assume $\g' =
1$ without loss of generality since scaling a decision function
$f$ by a positive constant does not alter the
induced classifier. However, alternate parametrizations such as
$\ind{y=1}\phi((1-\rho) t) + \ind{y=-1}\beta \phi(-\rho t)$, $\rho \in
(0,1)$, might be desirable in some situations.

A common motivation for uneven margin losses is
classification with an unbalanced training data set.
In unbalanced data, one class has (often substantially)
more representation than the other, and margin losses have been observed
to perform poorly in such situations. Weighted
margin losses, which have the form
$\a'\ind{y=1} \phi(t) + (1-\a')\ind{y=-1}
\phi(-t)$,
are often used as a heuristic for unbalanced data.
However, as \citet{steinwart07} notes,
there is no reason why the $\a'$ that yields good performance on
unbalanced data will be the desired cost parameter $\a$.
In other words, this heuristic typically results in
losses that are not $\a$-CC.

The parameter $\g$ offers another means to accommodate unbalanced data.
Such losses have previously been explored in the context of specific
algorithms, including the perceptron \citep{herbrich02}, boosting
\citep{vasconcelos07}, and support vector machines \citep{yang09,
shawetaylor03}. Uneven margins ($\g \ne 1$) have been found to yield
improved empirical performance in classification problems involving
label-dependent costs and/or unbalanced data.

Prior work involving uneven margin losses
has not addressed the issue of whether these losses
are CC or $\a$-CC. The following result
clarifies the issue for convex $\phi$.

\begin{cor} \label{cor:cvxuml}
Let $\phi$ be convex and differentiable at 0, $\beta, \g
> 0$ and let $L$, $L_\a$ be the associated uneven
margin losses as in Definition \ref{def:uml}.
The following are equivalent:
\bd
\item[(a)] $L$ is CC
\item[(b)] $L_\a$ is $\a$-CC
\item[(c)] $\beta = \frac1{\g}$ and $\phi'(0) < 0$.
\ed
\end{cor}
\begin{proof}
The equivalence of (a) and (b) follows from Theorem
\ref{thm:HLaa}, and the equivalence of (b) and (c) follows
from Theorem \ref{thm:diff}.
\end{proof}
This result implies that for any $\a \in (0,1)$
and $\g > 0$,
$$
L_\a(y,t) = (1-\a) \ind{y=1} \phi(t) +
\frac{\a}{\g} \ind{y=-1} \phi(-\g t)
$$
is $\a$-CC provided $\phi$ is convex and $\phi'(0)<0$.
Thus, $\g$ is a parameter that
can be tuned as needed, such as for unbalanced data,
while the loss remains $\a$-CC. Figure \ref{fig:cvxpart}
displays the partial losses for
three common $\phi$ and for three values of $\g$.
If $\phi$ is not convex, then uneven margin
losses can still be $\a$-CC, but the necessary
relationship between $\beta$ and $\g$ may be
different from that given by Corollary \ref{cor:cvxuml}.
An example is given below where $\phi$ is a sigmoid.

\begin{figure}
\centering
\includegraphics[width = 1.0\textwidth]{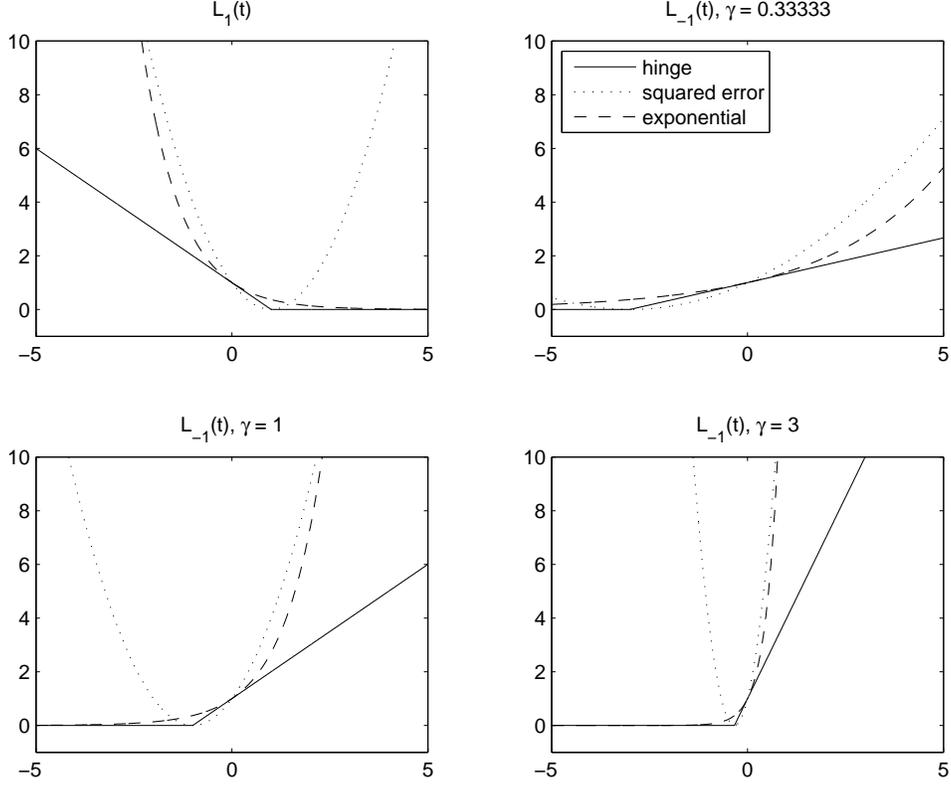}
\caption{\label{fig:cvxpart}Partial losses of an
uneven margin loss, for three common $\phi$
(hinge, squared error, and exponential) and three values
of $\g$.}
\end{figure}

To illustrate the general theory developed in Sec. 2,
four examples of uneven margin losses, corresponding to
different $\phi$, are now considered in detail.
The first three are convex, while the fourth is not.
In each case, the
primary effort goes in to computing $H_L(\eta) =
C_L^-(\eta) - C_L^*(\eta)$. Given $H_L$,
$H_{L_\a,\a}$ is determined by Eqn.
(\ref{eqn:HLaa}), and $\nu_{L_\a,\a}$ by
Eqns. (\ref{eqn:nu1}) and (\ref{eqn:nu2}).
For the convex $\phi$, all of which satisfy
$\phi(0) = 1$, $C_L^-(\eta) = \eta + \frac1{\g}(1-\eta)$
by Theorem \ref{thm:conv}, part 2.

\subsection{Uneven hinge loss}
Let $\phi(t) = (1-t)_+$, where $(s)_+ =
\max(0,s)$.  Then
$$
L(y,t) = \ind{y=1}(1 - t)_+ + \ind{y=-1} \frac1{\g} (1 + \g t)_+
$$
and
\beas
C_L(\eta, t) &=& \eta (1 - t)_+ +
\frac{1-\eta}{\g} (1 + \g t)_+ \\
&=& \left \{
\begin{array}{ll}
\eta(1-t), & t \le -\mbox{$\frac1{\g}$} \\
\eta (1-t) + \sfrac{1-\eta}{\g}(1 + \g t), &
\sfrac{-1}{\g} < t < 1 \\
\sfrac{1-\eta}{\g} (1+ \g t), & t \ge 1.
\end{array}
\right.
\eeas
Since $C_L$ is piecewise linear and continuous,
we know
$C_L^*(\eta)$ is the value of $C_L(\eta,t)$
when $t$ is one of the two knot locations.
Thus
\beas
C_L^*(\eta) &=& \min(\eta(1+\sfrac{1}{\g}),
\sfrac{1-\eta}{\g}(1+\g)) \\
&=& \sfrac{1+\g}{\g} \min(\eta,1-\eta)
\eeas
and
\beas
H_L(\eta) &=& \eta + \sfrac{1}{\g}(1-\eta)
- \sfrac{1+ \g}{\g}  \min (\eta, 1-\eta) \\
&=& \left\{
\begin{array}{ll}
2\eta -1, & \eta \ge \half \\
\sfrac{1-2\eta}{\g}, & \eta<\half.
\end{array}
\right.
\eeas

Now $\HLaa(\eta)$ is
given by Eqn. (\ref{eqn:HLaa}), and $\nuLaa$
is given by Eqns. (\ref{eqn:nu1}) and
(\ref{eqn:nu2}).  For the hinge case these
expressions simplify considerably:
$$
\HLaa(\eta) = \bcase
\eta - \a, & \eta \ge \a \\
\sfrac{\a-\eta}{\g}, & \eta < \a.
\ecase
$$
Expressions for $\nuLaa$ are given below.
Figure \ref{fig:hinge} shows $\HLaa$ and
$\nuLaa$ for three values of $\a$ and four
values of $\g$.

\begin{figure}
\centering
\includegraphics[width = 1.0\textwidth]{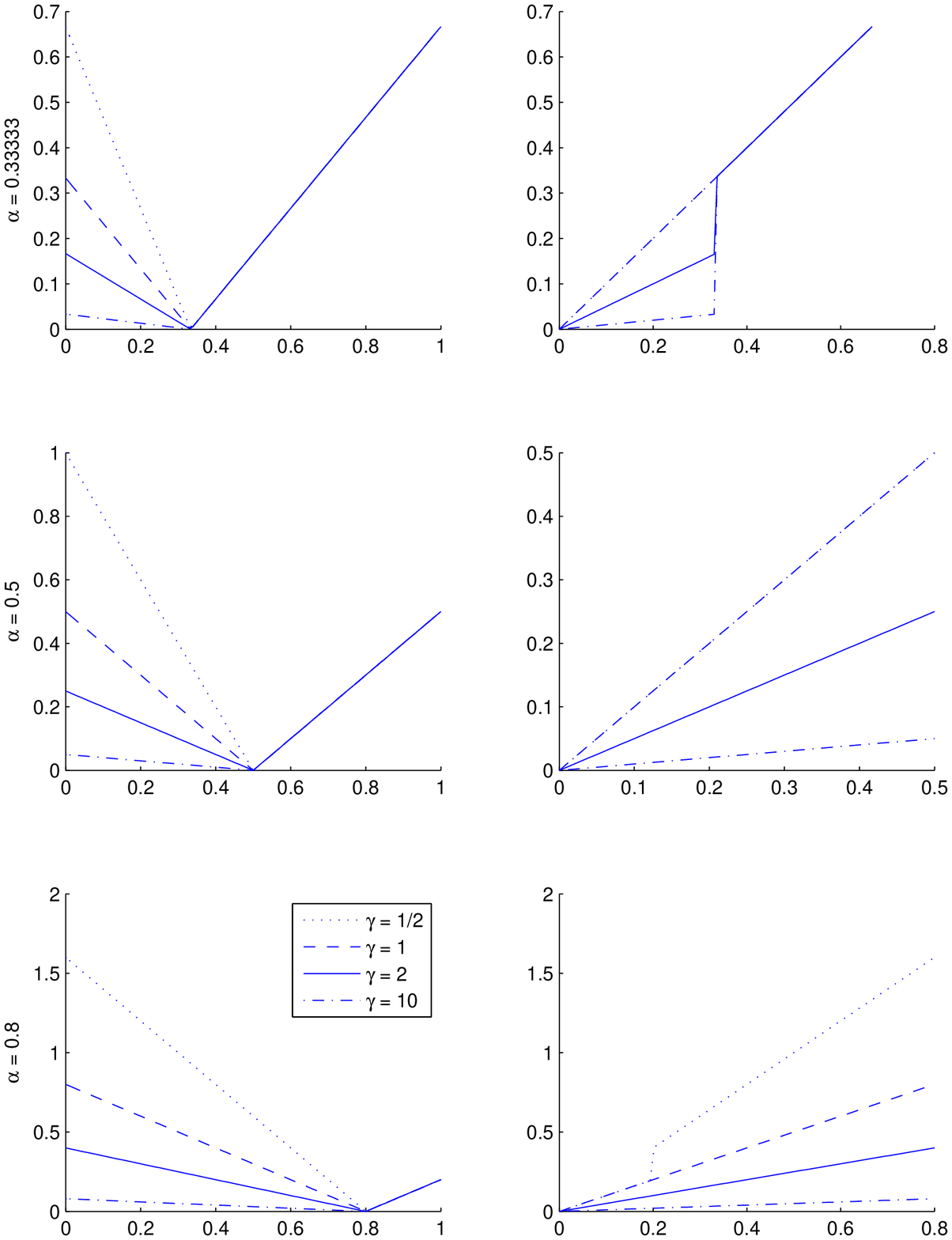}
\caption{{\bf Uneven hinge loss}.
$\HLaa$ (left column)
and $\nuLaa$ (right column) for three values of $\a$ and
four values of $\g$.
\label{fig:hinge}}
\end{figure}

These plots illustrate how $\nuLaa$ is
sometimes discontinuous at $\min(\a,1-\a)$.
We can characterize when $\nuLaa$ has a
discontinuity as follows.  From Eqn. (\ref{eqn:nu1}),
for $\a < \frac12$,
$$
\nuLaa(\eps) = \bcase
\min(\eps, \sfrac{\eps}{\g}), & 0 \le \eps \le \a \\
\eps, & \a < \eps \le 1-\a.
\ecase
$$
This is discontinuous at $\a$ iff $\g>1$
By Eqn. (\ref{eqn:nu2}), for $\a > \frac12$,
$$
\nuLaa (\eps) = \bcase
\min(\eps, \sfrac{\eps}{\g}), & 0 \le \eps \le 1-\a \\
\sfrac{\eps}{\g}, & 1-\a < \eps \le \a.
\ecase
$$
This is discontinuous at $1-\a$ iff $\g<1$.
If $\a=\frac12$, $\nuLaa$ is never
discontinuous.  In summary, $\nuLaa$ is discontinuous
at $\min(\a,1-\a)$ iff $(\a-\frac12)
(\g-1) < 0$.

\subsection{Uneven squared error loss}
Now let $\phi(t) = (1-t)^2$.  Then
$$
L(y,t) = \ind{y=1}(1-t)^2 +
\ind{y=-1} \frac1{\g} (1+\g t)^2
$$
and
$$
C_L(\eta,t) = \eta (1-t)^2 + \frac{1-\eta}{\g}
(1+\g t)^2.
$$
The minimizer of $C_L(\eta,t)$ is
$$
t^* = \frac{2 \eta-1}{\eta + \g(1-\eta)}.
$$
This yields (after some algebra)
$$
C_L^*(\eta) = C_L(\eta,t^*) = \frac{(1+\g)^2}{\g}
\cdot \frac{\eta (1-\eta)}{\eta + \g (1-\eta)},
$$
and therefore
$$
H_L(\eta) = \eta + \frac1{\g} (1 - \eta)
- \frac{(1+\g)^2}{\g} \cdot
\frac{\eta(1-\eta)}{\eta + \g(1-\eta)}.
$$
Figure \ref{fig:sqrerr} show plots of $\HLaa$ and
$\nuLaa$ for various values of $\a$ and $\g$.
We see again evidence that $\nuLaa$ can be
discontinuous at $\min(\a,1-\a)$.

As in the other example, we have not indicated
$\psi_{L_\a,\a}$.  Yet it can easily
be visualized as the largest
convex minorant of $\nuLaa$.
In many cases, $\nuLaa$ is actually convex
and hence equals $\psi_{L_\a,\a}$.  The
same comment applies to the hinge
and exponential examples.

\begin{figure}
\centering
\includegraphics[width = 1.0\textwidth]{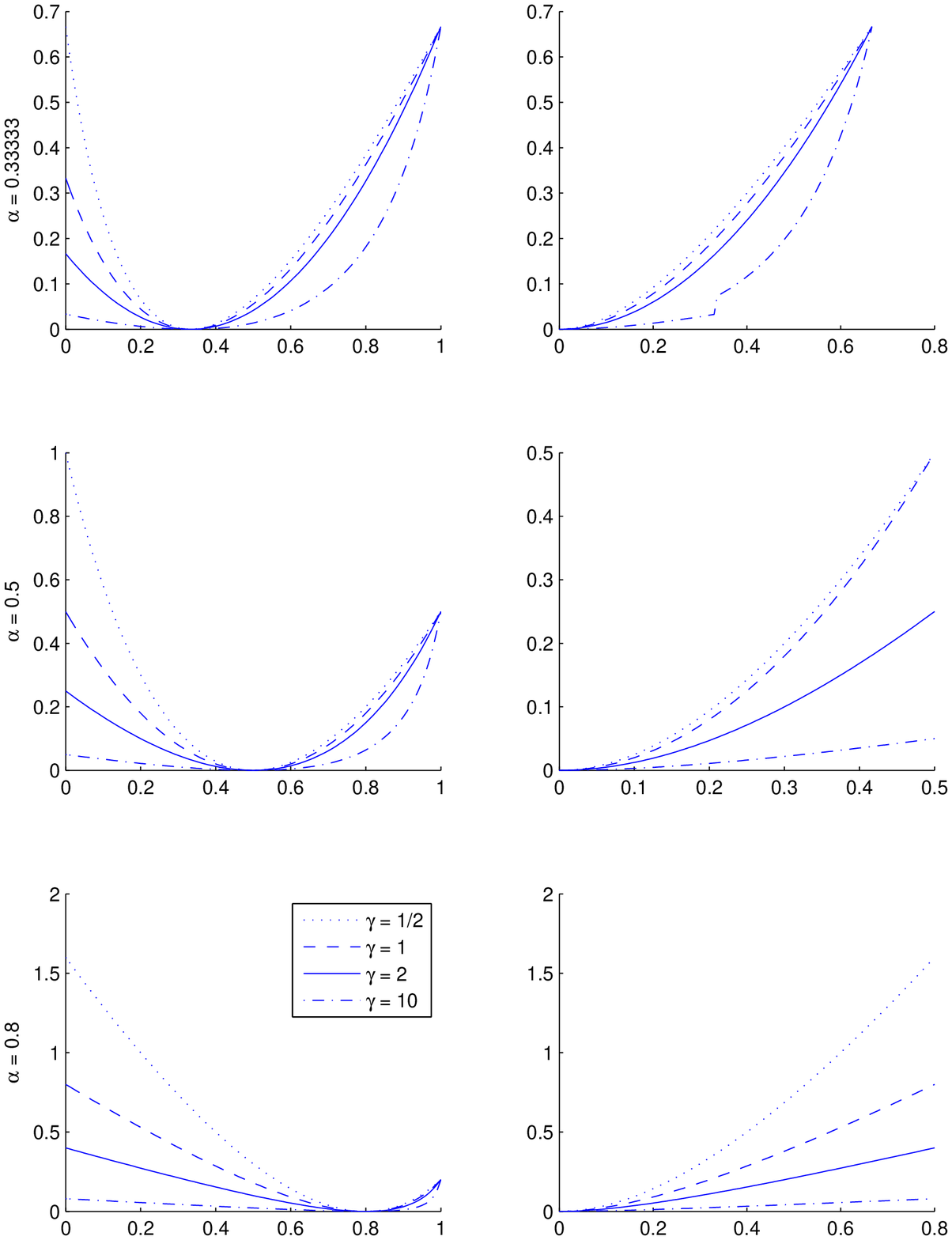}
\caption{{\bf Uneven squared error loss}.
$\HLaa$ (left column)
and $\nuLaa$ (right column) for three values of $\a$ and
four values of $\g$.
\label{fig:sqrerr}}
\end{figure}

\subsection{Uneven exponential loss}
\label{sec:exp}

Now let $\phi(t) = e^{-t}$ and consider
$$
L(y,t) = \ind{y=1} e^{-t} + \ind{y=-1} \frac1{\g} e^{\g t}.
$$
Then
$$
C_L(\eta,t) = \eta e^{-t} + \frac{1-\eta}{\g}
e^{\g t}
$$
is minimized by
$$
t^* = \frac1{1+\g} \ln \left(
\frac{\eta}{1-\eta} \right),
$$
yielding
$$
C_L^*(\eta) = C_L(\eta,t^*) = \eta \left(
\frac{1-\eta}{\eta} \right)^{\sfrac{1}{1+\g}}
+ (1-\eta) \left( \frac{\eta}{1-\eta}
\right)^{\sfrac{\g}{1+\g}}.
$$
Figure \ref{fig:exp} shows plots of $\HLaa$
and $\nuLaa$ for various $\a$ and $\g$.

\begin{figure}
\centering
\includegraphics[width = 1.0\textwidth]{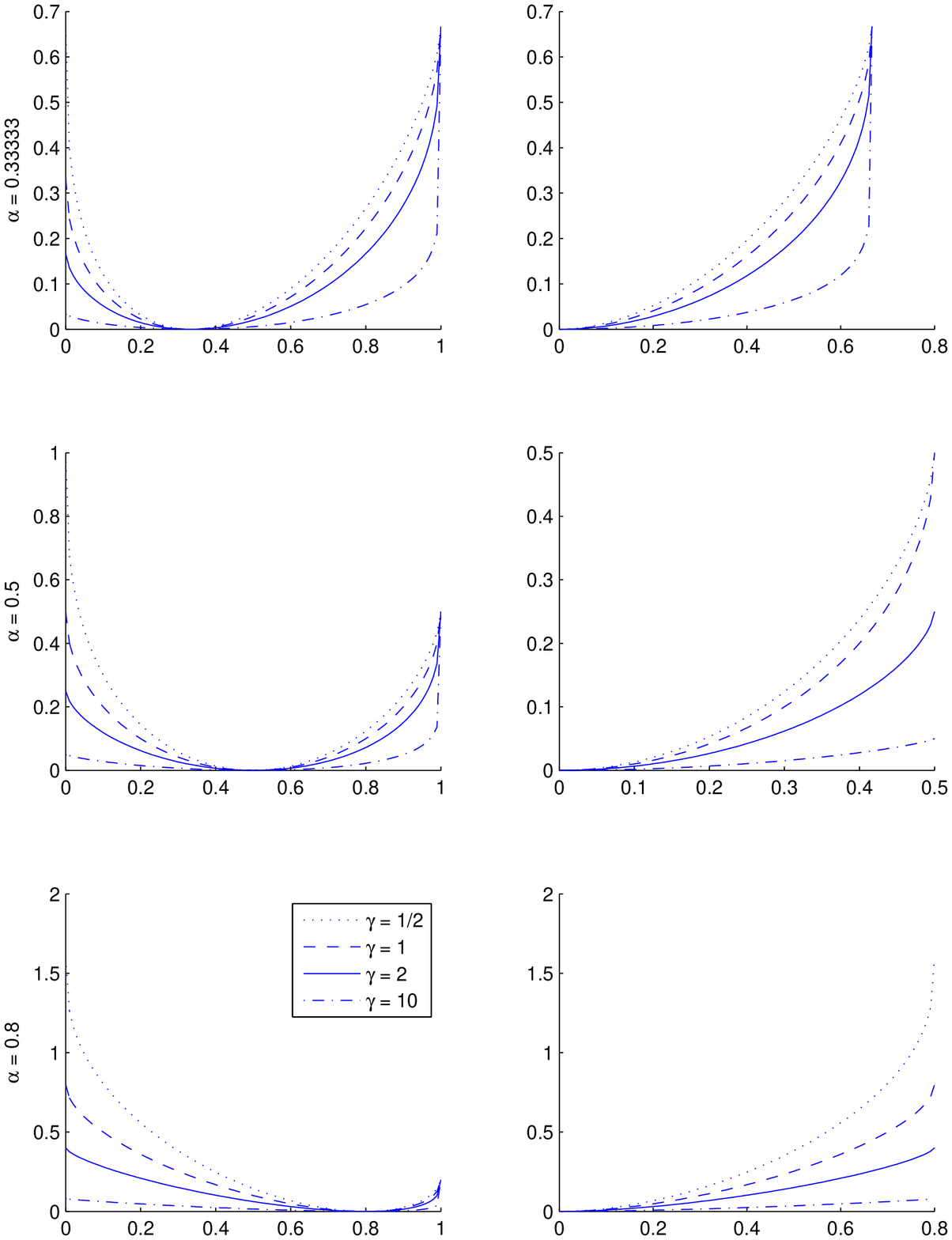}
\caption{{\bf Uneven exponential loss}.
$\HLaa$ (left column)
and $\nuLaa$ (right column) for three values of $\a$ and
four values of $\g$.
\label{fig:exp}}
\end{figure}

\subsection{Uneven sigmoid loss}
\label{sec:sig}
Finally we consider a nonconvex $\phi$, namely
the sigmoid function $\phi(t) =
1/(1+e^t)$.  For concreteness, we fix
$\g=2$ and study
$$
L(y,t) = \ind{y=1} \frac1{1+e^t} + \ind{y=-1}
\frac1{2} \frac1{1+e^{-2 t}}.
$$
General $\gamma$ will be discussed at the end.

Since $\phi$ is not convex, we cannot conclude
$L$ is CC.  In fact, we will show that
$L$ is $\a$-CC for $\a = (3+4\sqrt{2})/23
\approx 0.37639$.

Figure \ref{fig:sigloss} shows
$$
C_L(\eta,t) = \eta \frac{1}{1+e^{-t}} +
\frac{1-\eta}{2} \frac1{1+e^{2t}}
$$
as a function of $t$, for six different $\eta$.
These graphs are useful in understanding
$\CLam(\eta)$ and $C_L^*(\eta)$.
When $\eta < \frac12$, it can be shown that
$C_L(\eta,t)$ has a single local minimum
and a single local maximum.  When $\eta \ge
\frac12$, on the other hand, $C_L(\eta,t)$
is strictly decreasing.  Let
$t_-(\eta)$ denote the local minimizer when
$\eta < \frac12$.  This function can
be expressed in closed form.  See Appendix
\ref{app:sig} for these and other details.

\begin{figure}
\centering
\includegraphics[width = 1.0\textwidth]{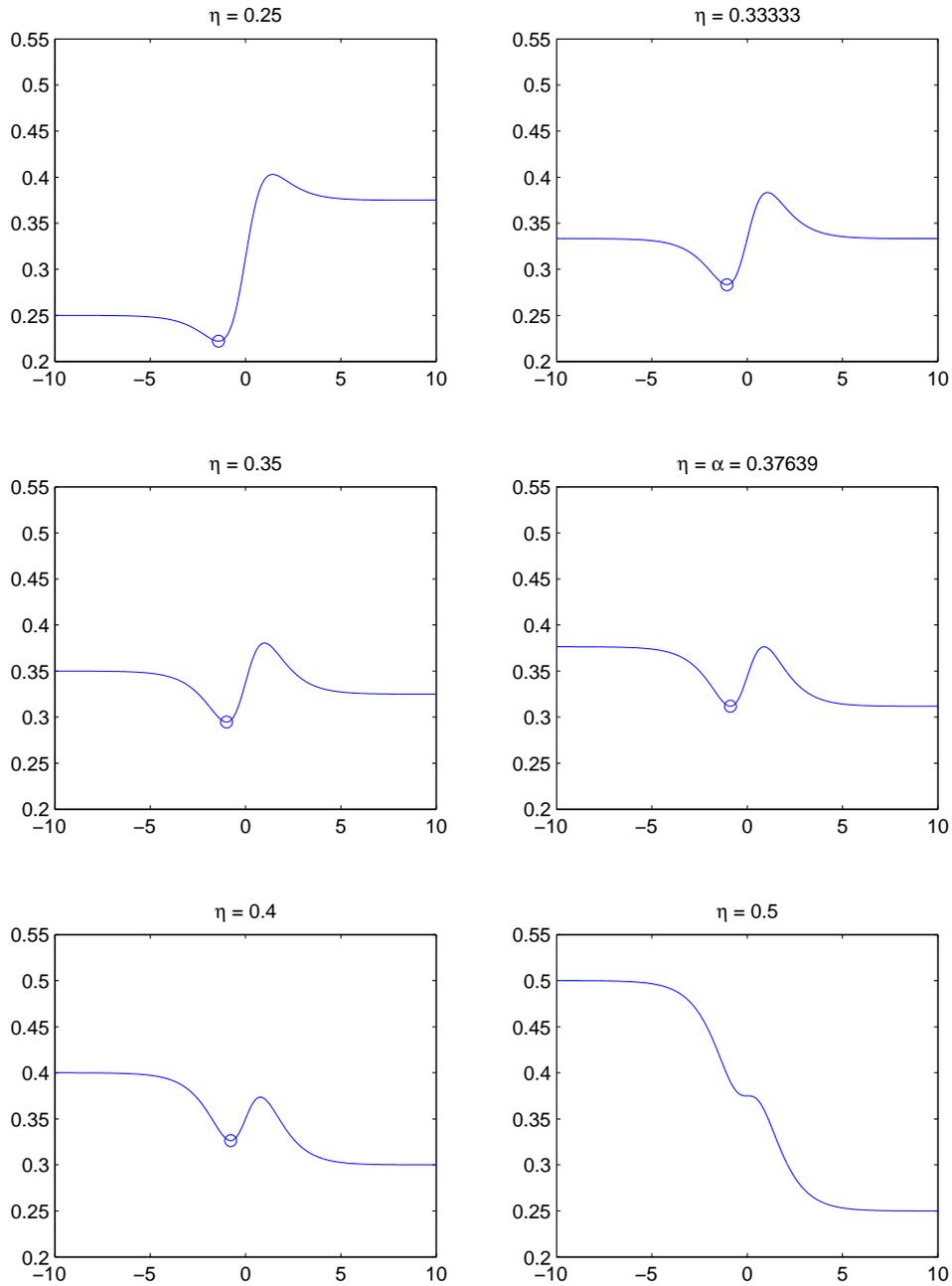}
\caption{{\bf Uneven sigmoid loss with $\g = 2$}.  $C_L(\eta,t)$
is graphed as a function of $t$ for six
values of $\eta$.  The circles indicate
$(t_-(\eta),C_L(\eta, t_-(\eta)))$.
\label{fig:sigloss}}
\end{figure}

First, we determine $C_L^*$.  The infimum of $C_L(\eta,t)$ over $t \in
\reals$
is either $C_L(\eta, t_-(\eta))$ or $C_L(\eta, \infty)
= (1-\eta)/2$.  As indicated by Figure \ref{fig:sigloss},
$C_L(\eta, t_-(\eta)) = C_L(\eta, \infty)$ when
$\eta = \a = (3 + 4\sqrt{2})/23 \approx 0.37639$.
See Appendix \ref{app:sig} for proof of this fact.
When $\eta < \a$, $C_L^*(\eta) = C_L(
\eta, t_-(\eta))$, and when $\eta \ge \a$,
$C_L^*(\eta) = C_L(\eta, \infty) = (1-\eta)/2$.
Thus,
$$
C_L^*(\eta) = \bcase
C_L(\eta,t_-(\eta)), & \eta < \a \\
\sfrac{1-\eta}2, & \eta \ge \a. \ecase
$$

Next, consider $\CLam$.  When $\eta < \a$,
$\CLam(\eta)$ is either $C_L(\eta,0) =
(1+\eta)/4$ or $C_L(\eta,\infty) = (1-\eta)/2$.
Since $\frac{1+\eta}4 < \frac{1-\eta}2 \iff
\eta < \frac13$, we have $\CLam(\eta) =
(1+\eta)/4$ for $0\le \eta \le \frac13$
and $\CLam(\eta) = (1-\eta)/2$ if
$\frac13 <\eta < \a$.  When $\eta \ge \a$,
$\CLam(\eta) = C_L(\eta,t_-(\eta))$ when
$ \a \le \eta \le \frac12$, and $\CLam(\eta) =
C_L(\eta,0) = (1+\eta)/4$ for $\eta\ge\frac12$.
In summary,
$$
\CLam(\eta) = \bcase
\sfrac{1+\eta}{4}, & 0 \le \eta \le \frac13
\mbox{ or } \eta \ge \half \\
\sfrac{1-\eta}{2}, & \third < \eta < \a \\
C_L(\eta,t_-(\eta)), & \a < \eta < \half.
\ecase
$$
Now $\HLa(\eta) =
\CLam(\eta) - C_L^*(\eta)$. See Figure \ref{fig:sigH} for
plots of these quantities. This is our first example where
$\HLa$ is not convex.

\begin{figure}
\centering
\includegraphics[width = 1.0\textwidth]{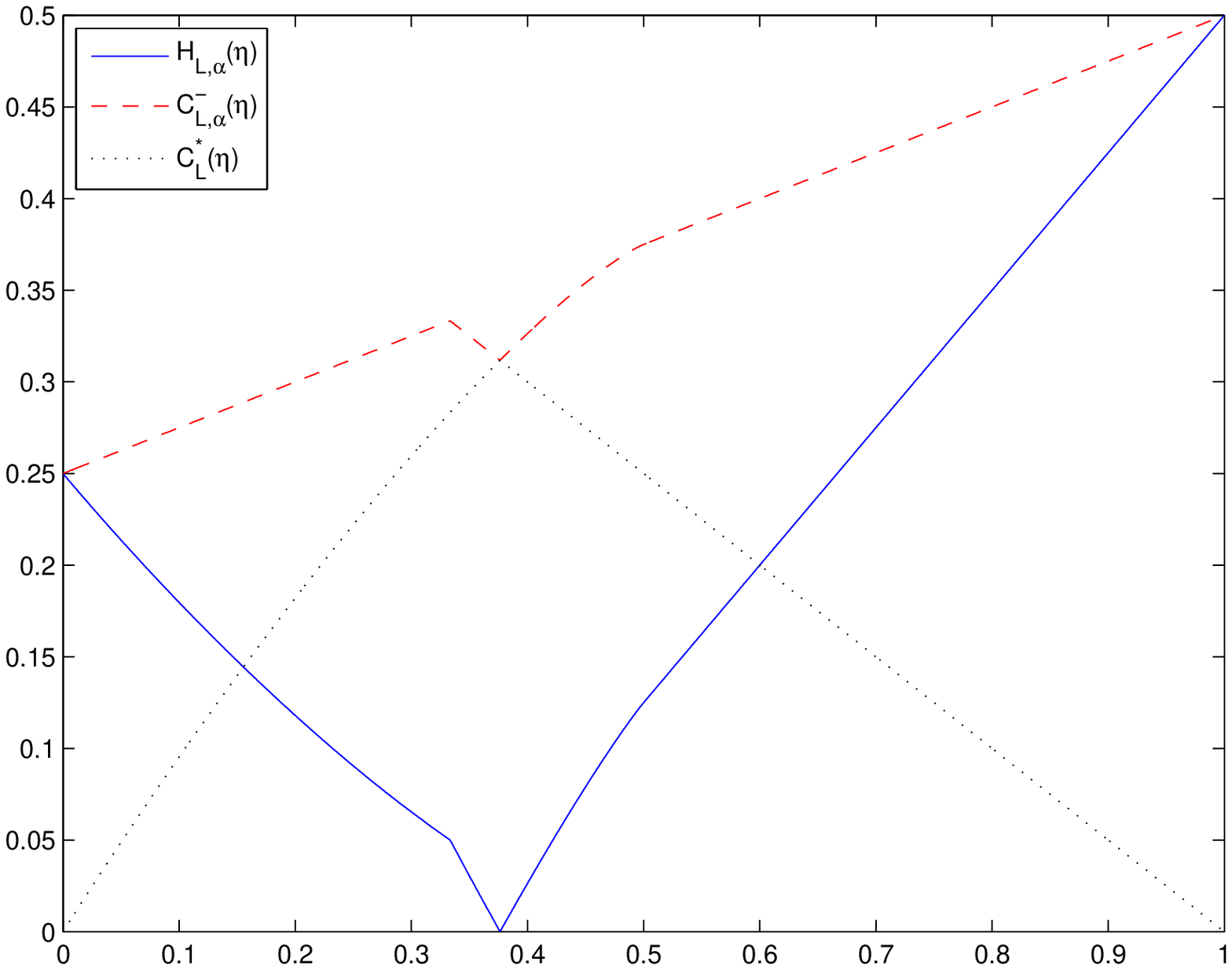}
\caption{{\bf Uneven sigmoid loss with $\g = 2$}. Plots of $\HLa$,
$\CLam$, and $C_L^*$ for $\a = (3+4\sqrt{2})/23 \approx 0.37639$.
\label{fig:sigH}}
\end{figure}

Finally, the preceding discussion can be extended to
arbitrary $\g>0$.  For every $\g > 0$
there is a unique $\a = \a(\g)\in(0,1)$ such that
\begin{equation}\label{eqn:siggen}
L(y,t) = \ind{y=1}\frac1{1+e^t} +
\ind{y=-1} \frac1{\g}
\frac1{1+e^{-\g t}}
\end{equation}
is $\a$-CC.
The relationship
between $\a$ and $\g$ is shown in Figure
\ref{fig:ag}.  Calculation of this curve is
discussed in Appendix \ref{app:sig}.  In the
appendix we show that $\a(\sfrac{1}{\g})
=1-\a(\g)$, which explains the sigmoidal
shape of $\a$ as a function\footnote{We investigated whether $\a(\g) =
1/(1+e^{c\ln\g})$ for some $c>0$,
but evidently it does not.} of $\ln \g$.

Now suppose $\a' \in (0,1)$ is the desired cost asymmetry.  By Theorem
\ref{thm:HLaa}, for $L$ in Eqn. (\ref{eqn:siggen}), $L_{1-\a(\g)}$ is CC,
and therefore $L_{(1-\a(\gamma))\a'}$ is $\a'$-CC. This is a family of
losses, indexed by $\g > 0$, all of which are $\a'$-CC.

\begin{figure}
\centering
\includegraphics[width = 1.0\textwidth]{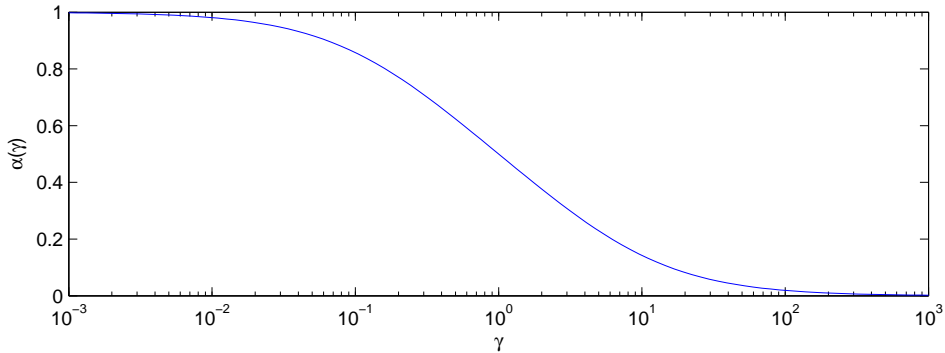}
\caption{{\bf Uneven sigmoid loss}.  Plot of the unique value of
$\a = \a(\g)$ such that the
uneven sigmoid loss with parameter $\g>0$ (Eqn.
(\ref{eqn:siggen})) is $\a$-CC.
\label{fig:ag}}
\end{figure}

\section{Discussion}
\label{sec:disc}

The results of \citet{bartlett06} concerning surrogate regret bounds and
classification calibration are generalized to label-dependent
misclassification costs and arbitrary losses.  Some differences that
emerge in this more general framework are that $\HLa(\eta)$ is in general
not symmetric about $\eta = \frac12$, and $\nuLa(\eps)$ is potentially
discontinuous at $\eps = \min(\a,1-\a)$.  The framework of
\citet{steinwart07} is also applied. Although his notion of calibration is
not always equivalent to the one adopted here, that approach based on
calibration functions nonetheless leads to the same surrogate regret
bounds.

The class of uneven margin losses are examined in some detail.  We hope
these results provide guidance to future work with such losses, as our
theory explains how to ensure $\a$-classification calibration for any
margin asymmetry parameter $\g > 0$.  For example, Adaboost is often
applied to heavily unbalanced datasets where misclassification costs are
label-dependent, such as in cascades for face detection
\citep{violajones}.  It should be possible to generalize Adaboost to have
an uneven margin (to accommodate unbalanced data) while being
$\a$-classification calibrated for any $\a \in (0,1)$.  In particular, the
uneven exponential loss from Sec. \ref{sec:exp} can be optimized by the
functional gradient descent approach.  In fact, \citet{vasconcelos07}
developed such an algorithm for the special case $\g = \a/(1-\a)$, but did
not identify the generalization to arbitrary $\g$.

Our theory also sheds light on the support vector machine with uneven 
margin. \citet{yang09} describe an implementation of this algorithm, but 
they allow for both $\beta$ and $\g$ to be free parameters. Our Corollary
\ref{cor:cvxuml} constrains $\beta = 1/\g$ for classification calibration, 
which eliminates a tuning parameter.

In closing, we mention two additional directions for future work.  First, 
an interesting problem related to uneven margin losses is that of 
surrogate tuning, which in this case is the problem of tuning the 
parameter $\g$ to a particular dataset. \citet{nock09} have recently 
described a data-driven approach to surrogate tuning of 
classification-calibrated ($\a = \frac12$) losses.  Second, our regret 
bounds should be applicable to proving the cost-sensitive consistency of 
algorithms based on surrogate losses.

\section*{Acknowledgements}

This work was supported in part by NSF Grants CCF-0830490 and CCF-0953135. 

\appendix

\section{Lemmas}
LSC and USC abbreviate lower semi-continuous and upper
semi-continuous.

\begin{lemma}
\label{lemmaX}
Let $L$ be a loss, $\a \in (0,1)$, and recall $B_\a
= \max(\a,1-\a)$.
\bd
\item[1.] (a) For any $\eta \in [0,1]$, $C_\a^*(\eta)
= C_\a(\eta, \eta-\a)$. (b) For any $\eta \in [0,1],
t \in \reals$, $C_\a(\eta,t) - C_\a^*(\eta) =
\ind{\sign(t) \ne \sign(\eta-\a)}|\eta-\a|$.
(c) $R_\a^* = R_\a(\eta-\a)$. (d) For any $f
\in \sF$,
$$
R_\a(f)-R_\a^* = E_X[
\ind{\sign(f(X)) \ne \sign(\eta(X)-\a)}|\eta(X)
-\a|].
$$
\item[2.] (a) $C_L^*(\eta)$ is concave on
$[0,1]$. (b) $\CLam(\eta)$ is concave on
$[0,\a)$ and on $(\a,1]$.
\item[3.] (a) $C_L(\eta)$ is continuous on $[0,1]$.
(b) $\CLam(\eta)$ and $\HLa(\eta)$ are continuous
on $[0,1] \backslash \{\a\}$. (c) If $L$ is
$\a$-CC, then $\CLam$ and $\HLa$ are continuous
on $[0,1]$.
\item[4.] $\HLa(\a) = \nuLa(0) = \muLa(0)
= \psiLa(0) = 0$.
\item[5.] $\nuLa$ and $\muLa$ are LSC on
$[0,B_\a]$. $\psiLa$ is continuous on $[0,B_\a]$.
\ed
\end{lemma}

\begin{proof}
{\bf 1.} For $\eta \in [0,1]$, $C_\a(\eta,t)
= (1-\a)\eta\ind{t\le0} + \a(1-\eta)\ind{t>0}$ is minimized by any $t$
such that
$\sign(t) = \sign((1-\a)\eta - \a(1-\eta))
= \sign(\eta - \a)$. Therefore $C_\a(\eta,
\eta-\a) = C_\a^*$. This gives (a). It also
implies
\beas
\lefteqn{C_\a(\eta,t) - C_\a^*(\eta)} \\
&=& (1-\a)\eta\ind{t\le0} + \a(1-\eta)\ind{t>0}
- [(1-\a)\eta\ind{\eta\le\a} + \a(1-\eta)
\ind{\eta>\a}] \\
&=& \ind{\sign(t) \ne \sign(\eta-\a)}|\eta-\a|,
\eeas
which is (b). Part (c) now follows from (a) and
$R_\a^* = E_X[C_\a^*(\eta(X))] =
E_X[C_\a(\eta(X), \eta(X)-\a)] = R_\a(
\eta-\a)$, while (d) follows from (b) and
\beas
R_\a(f)-R_\a^* &=& E_X[C_\a(\eta(X), f(X))
-C_\a^*(\eta(X))]  \\
&=& E_X[\ind{\sign(f(X))\ne
\sign(\eta(X)-\a)}|\eta(X)-\a|].
\eeas

{\bf 2.} Since $C_L^*(\eta) = \inf_{t\in\reals}
\eta\Lp(t) + (1-\eta)\Lm(t)$, it is the
infimum of affine functions and therefore concave.
For $\eta < \a$, $\CLam(\eta) = \inf_{t\ge 0}
C_L(\eta,t)$ which is also concave by the same
reasoning. A similar argument applies when $\eta > \a$.

{\bf 3.} Since $C_L^*(\eta)$ is concave
on $[0,1]$, it is continuous on $(0,1)$ by
Theorem 10.1 of \citet{rock70}. By Theorem 10.2
of the same, $C_L^*$ is LSC at 0 and 1.
Let us argue that $C_L^*$ is USC at 1, the case
of 0 being similar. Thus, let $\eps > 0$ and let $t_\eps
\in \reals$ such that $\Lp(t_\eps) \le
C_L^*(1) + \frac{\eps}2$. If $\Lm(t_\eps) = 0$, then
for any $\eta \in [0,1)$, $C_L^*(\eta) \le C_L(\eta,t_\eps) = \eta
\Lp(t_\eps)
\le \Lp(t_\eps) \le C_L^*(1) + \eps$. Suppose $\Lm(t_\eps)
> 0$.
If $\eta$ is such that
$1-\frac{\eps}{2\Lm(t_\eps)} \le \eta < 1$, then
$C_L^*(\eta) \le \eta \Lp(t_\eps) + (1-\eta)
\Lm(t_\eps) \le C_L^*(1) + \eps$. Thus $C_L^*$ is USC at 1. This
establishes (a).

For (b), continuity of $\CLam$ on $[0,1]
\backslash \{\a\}$ follows by a similar argument
as (a). Continuity of $\HLa$ then follows immediately.

It remains to show that $\CLam$, and hence $\HLa$,
is continuous at $\a$ when $L$ is $\a$-CC.
First note that $\CLam$ is LSC at $\a$
because $\CLam(\a) = C_L^*(\a)$, $\CLam(\eta)
\ge C_L^*(\eta)$ for all $\eta \in [0,1]$,
and from parts (a) and (b).

We now show $\CLam$ is USC at $\a$ when $L$
is $\a$-CC. Let $\eps > 0$. Since $C_L^*$ is
continuous at $\a$, there exists $\delta' > 0$ such
that $|C_L^*(\eta)-C_L^*(\a)| < \frac{\eps}3$
whenever $|\eta - \a| < \delta'$. Let $\delta_\a
= \frac12 \min(\a,1-\a), M = \max(\Lp(0),
\Lm(0))$, and set $\delta = \min(\delta', \delta_\a,
\frac{\eps}3 \cdot \frac{\delta_\a}{2M})$.
Now suppose $|\eta - \a| < \delta$, $\eta \ne \a$.
Then
\beas
\CLam(\eta)-\CLam(\a) &=& \CLam(\eta) -
C_L^*(2\a-\eta) + C_L^*(2\a-\eta) -
C_L^*(\a) \\
&\le& \CLam(\eta)-C_L^*(2\a-\eta) + \frac{\eps}3,
\eeas
since $|(2\a-\eta)-\a| = |\eta-\a| < \delta
\le \delta'$. Since $L$ is $\a$-CC,
there exists $t^*$, depending possibly on $\eta$ and
$\eps$, such that $t^*((2\a-\eta)-\a) \ge 0$
and $C_L(2\a-\eta,t^*) \le C_L^*(2\a-\eta) +
\frac{\eps}3$. We may further stipulate
$C_L(2\a-\eta,t^*) \le C_L(2\a-\eta,0)$ which will
be needed later. Notice $t^*((2\a-\eta)-\a) \ge 0
\iff t^*(\eta-\a) \le 0$, which is also
used later. Now $\CLam(\eta)-C_L^*(2\a-\eta)
\le \CLam(\eta)-C_L(2\a-\eta,t^*)
+ \frac{\eps}3$.  Thus far we have shown
$\CLam(\eta)-\CLam(\a) \le \CLam(\eta)
-C_L(2\a-\eta,t^*) + \frac{2\eps}3$
for $|\eta-\a| < \delta, \eta \ne \a$.

Now consider
\beas
\lefteqn{\CLam(\eta)-C_L(2\a-\eta,t^*)
= \inf_{\topp} C_L(\eta,t) - C_L(2\a-\eta,t^*)}
\\
&\le& C_L(\eta,t^*) - C_L(2\a-\eta,t^*) \\
&=&\eta \Lp(t^*) + (1-\eta)\Lp(t^*) -
[(2\a-\eta)\Lp(t^*) + (1-(2\a-\eta))\Lm(t^*)] \\
&=& 2[\Lp(t^*)(\eta-\a) + \Lm(t^*)(\a-\eta)]
\\
&\le& 2[\Lp(t^*) + \Lm(t^*)]|\eta-\a|.
\eeas
To bound this quantity, observe
\beas
M&=& \max(\Lp(0), \Lm(0)) \\
&\ge& C_L(2\a-\eta,0) \\
&\ge& C_L(2\a-\eta,t^*) \\
&=& (2\a-\eta)\Lp(t^*) + (1-(2\a-\eta))
\Lm(t^*) \\
&\ge& \frac{\a}2\Lp(t^*) + \frac{1-\a}2
\Lm(t^*) \\
&\ge& \delta_\a(\Lp(t^*) + \Lm(t^*)).
\eeas
To see the next to last inequality, recall
$|\eta-\a| < \delta \le \delta_\a = \frac12
\min(\a,1-\a)$. Then $2\a-\eta
= \a + (\a-\eta) \ge \frac{\a}2$ and
$1-(2\a-\eta) = 1-\a + (\eta-a) \ge
\frac{1-\a}2$
We now have $\CLam(\eta)-C_L(2\a-\eta,t^*)
\le \frac{2M}{\delta_\a}|\eta-\a| < \frac{\eps}3$.

We have shown that for all $\eps > 0$, there exist $\d > 0$ such that for
all
$\eta \in [0,1]$ with $|\eta - \a| < \d$ and
$\eta \ne \a$,
$$
\CLam(\eta) - \CLam(\a) < \eps.
$$
Therefore $\CLam$ is USC, and hence continuous, at $\a$.

{\bf 4.} $\HLa(\a)=0$ because when $\eta=\a$,
the infimum defining $\CLam(\a)$ is unrestricted.
From this we have $\nuLa(0) =
\HLa(\a)=0$. Since $0 \le \muLa(0) \le
\nuLa(0)$ we deduce $\muLa(0)=0$.
Finally, $\psiLa(0)=0$ because $\psiLa =
\nuLa^{**}$, $\nuLa(0)=0$, and $\nuLa$ is
nonnegative.

{\bf 5.} From 3, $\HLa$ is continuous except possibly at $\a$.
Therefore $\nuLa$ is continuous except possibly at $0$ and
$b_\a:=\min(\a,1-\a)$. $\nuLa$ is LSC at $0$ because $\nuLa(0)=0$
and $\nuLa$ is nonnegative. $\nuLa$ is LSC at $b_\a$ because
$\nuLa(b_\a^-)=\nuLa(b_\a) \le \nuLa(b_\a^+)$, which follows from
the definition of $\nuLa$. Now lower semi-continuity of $\muLa$
follows from Lemma \ref{lemmaZ}.
\end{proof}

The following result generalizes Lemma A.7 of \citet{steinwart07}.
\begin{lemma}
\label{lemmaZ}
Let $\d:[0,B] \to [0,\infty)$ be a lower
semi-continuous function with $\d(0)=0$, and define
$\dt(\eps) = \inf_{\eps'\ge\eps} \d(\eps')$.
Then $\dt$ is lower semi-continuous and
$\dt^{**} = \d^{**}$.
\end{lemma}

\begin{proof}
Suppose $\dt$ is not LSC at $\eps \in [0,1]$.
Then there exists $\tau > 0$ and $\eps_1, \eps_2,
\ldots \to \eps$ such that for $i$ sufficiently large,
$\dt(\eps_i) \le \dt(\eps)-\tau$. Since $\dt$
is nondecreasing, we may assume $\eps_i < \eps$
for all $i$. If $\dt(\eps_i) \le \dt(\eps)-\tau$,
then there exists $\eps_i' \in [\eps_i,\eps)$
such that $\d(\eps_i') \le \dt(\eps)-\frac{\tau}2 \le \d(\eps) -
\frac{\tau}2$.
But $\eps_i' \to \eps$, which implies $\d$ is
not LSC at $\eps$, a contradiction.

To show $\dt^{**} = \d^{**}$, we need to show
$\overline{\co \Epi \dt} = \overline{\co \Epi \d}$. It
suffices to show $\co \Epi \dt = \co \Epi \d$.
Since $\dt \le \d$, clearly $\Epi \dt \subset
\Epi \d$ and therefore $\co \Epi \dt \subset
\co \Epi \d$. For the reverse inclusion, it suffices
to show $(\eps, \dt(\eps)) \in \co \Epi \d$
for all $\eps \in [0,B]$. We may assume
$\eps \in (0,B)$ since $\d(0) = \dt(0) = 0$
and $\d(B) = \dt(B)$. Thus let $\eps \in
(0,B)$. Since $\d$ is LSC, it achieves its
infimum over a compact set, and hence there exists
$\eps' \in [\eps,B]$ such that $\dt(\eps) =
\d(\eps')$. Since $(0,0), (\eps', \frac{\eps'}{\eps}
\dt(\eps)) \in \Epi(\d)$, it follows that
$$
\frac{\eps}{\eps'}(\eps', \frac{\eps'}{\eps} \dt(\eps)) +
\frac{\eps'-\eps}{\eps'}(0,0)
= (\eps, \dt(\eps)) \in\co\Epi\d,
$$
as was to be shown
\end{proof}

\section{Uneven Sigmoid Loss Details}
\label{app:sig}

We present a closed form
expression for $t_-(\eta)$, and
describe how to calculate
$\a(\g)$ from Sec. \ref{sec:sig}.

$t_-(\eta)$ is the value of $t$ that
satisfies $t<0$ and
$$
0 = \frac{\partial}{\partial t} C_L(\eta,t)
= \eta \phi'(t) - (1-\eta) \phi'(-2t).
$$
Using $\phi'(t) = -e^t/(1+e^t)^2$
and substituting $z =  e^t$, $z$ must
satisfy $z \in (0,1)$ and
$$
\eta \frac{z}{(1+z)^2} = (1-\eta)
\frac{z^{-2}}{(1+z^{-2})^2},
$$
or equivalently, $z \in (0,1)$ is a solution
of the quartic equation
\beas
0 &=& \eta z^4 - (1-\eta)z^3
+ 2(2\eta-1)z^2 - (1-\eta)z + \eta \\
&=& z^2(\eta z^2 - (1-\eta)z + 2(2\eta-1)
- (1-\eta)z^{-1} + \eta z^{-2}).
\eeas
Note $z=0$ is not the desired solution,
as it corresponds to $t = -\infty$.
Let $w = z + z^{-1}$, and observe
$w^2 = z^2 + 2 + z^{-2}$.  Then $z$
must satisfy
\beas
0 &=& \eta (z^2 + z^{-2}) - (1-\eta)
(z + z^{-1}) + 2(2\eta-1) \\
&=& \eta (w^2-2) - (1-\eta)w + 2(2\eta-1) \\
&=& \eta w^2 - (1-\eta)w + 2\eta-1.
\eeas
Therefore
$$
w=\frac{1-\eta + \sqrt{(1-\eta)^2 - 8\eta(\eta-1)}}
{2\eta}.
$$
We take the positive sign because only it gives
a positive $z$.  Now $z$ can be recovered from $w$.
Since $z^2-wz+1 = 0$ we get
$$
z=\frac{w-\sqrt{w^2-4}}2.
$$
We take the negative sign as we are seeking the smaller
of the two critical points.  It can
be shown (with algebra) that $w^2>4 \iff
\eta < \frac12$.  Finally, we have $t_-(\eta)
= \ln z$.

We now turn to characterization of $\a(\g)$.
Assume $\g > 1$.  $\a(\g)$ is the
value of $\eta$ such that
$$
\frac{1-\eta}{\g} = C_L(\eta,\infty)
= C_L(\eta,t) = \frac{\eta}{1+e^t} +
\frac{1-\eta}{\g} \frac1{1+e^{-\g t}}
$$
is satisfied by a unique $t$
with $-\infty < t < 0$.  Since $ C_L (\eta,
-\infty) = C_L(\eta,\infty) \iff \eta =
\frac1{1+\g}$, we must have $\eta >
\frac1{1+\g}$.  After substituting $z = e^t$
and simplifying, we seek
$\eta > \frac1{1+\g}$ such that
$$
\eta \g z^{\g} - (1-\eta)z + (\eta \g -1 +\eta) = 0
$$
is satisfied for a unique $z \in (0,1)$.  That is,
we need the curves
$p_\eta(z) := \eta \g z^\g$ and $q_\eta(z) :=
(1-\eta)z - (\eta \g -1+\eta)$ to
intersect exactly once on $(0,1)$.  Since
$p_\eta$ is a strictly increasing convex function and $q_\eta$ is a line
with positive
slope, this can happen in one of three ways: (a)
$p_\eta(0) > q_\eta(0)$ and $p_\eta(1) <
q_\eta(1)$, (b) $p_\eta(0) < q_\eta(0)$
and $p_\eta(1) > q_\eta(1)$, or (c) $q_\eta$
is tangent to $p_\eta$ at some $z\in(0,1)$.
(a) requires $\eta > 1/(1+\g)$ and $\eta < 1/(1+\g)$, which is impossible.
Similarly, (b) is impossible.
Thus, we must
have $p_{\eta}'(z) = q_{\eta}'(z)$
for some $z\in(0,1)$.

Summarizing up to this point, we seek $\eta
> \frac1{1+\g}$ and $z\in(0,1)$
such that
\begin{equation}
\label{eqn:ag1}
\eta \g z^\g = (1-\eta)z - (\eta\g-1+\eta)
\end{equation}
and
\begin{equation}
\label{eqn:ag2}
\eta \g^2 z^{\g-1} = 1-\eta.
\end{equation}
Dividing (\ref{eqn:ag1}) by (\ref{eqn:ag2})
and solving for $z$ gives
\begin{equation}
\label{eqn:agz}
z = \frac{\eta \g -1 + \eta}{1-\eta}
\frac{\g}{\g-1}.
\end{equation}
Substituting (\ref{eqn:agz}) into (\ref{eqn:ag2})
yields
\begin{equation}
\label{eqn:agbi}
\eta \left[\g^2 \left(\frac{\eta \g-1 + \eta}{1-\eta}
\frac{\g}{\g-1} \right)^{\g-1} + 1\right] =1.
\end{equation}
When $\g=2$, this simplifies to a quadratic
equation, leading to $\a(2) =
(3 + 4\sqrt{2})/23$.  More generally, notice
that for $\eta > \frac1{1+\g}$, the left-hand
side of (\ref{eqn:agbi}) is strictly increasing, and
thus $\eta=\a(\g)$ can be found with a bisection search.
The case $\g=1$ was treated by
\citet{bartlett06}, yielding $\a(1)=\frac12$.
When $\g<1$ we may appeal to symmetry.
Let us write $C_L^\g(\eta,t)$ to indicate the
dependence of $C_L$ on $\g$.  It is easily
shown that $C_L^{1/\g}(\eta, \g t)
= \g C_L^\g(1-\eta,-t)$, from
which it follows that $\a(\frac1{\g})
= 1-\a(\g)$.

\appendix

\bibliographystyle{plainnat}
\bibliography{surrLDarxiv}

\end{document}